\documentclass{article}  

\usepackage{amssymb, amsmath, amsthm}
\usepackage[hidelinks]{hyperref}
\usepackage[capitalize, noabbrev]{cleveref}
\usepackage[nocompress]{cite}

\usepackage{graphicx}

\newtheorem{lemma}{Lemma}
\newtheorem{proposition}{Proposition}
\newtheorem{theorem}{Theorem}

\newcommand{\R}{\mathbb{R}}

\usepackage{enumitem}
\crefalias{enumi}{theorem}
\crefalias{enumii}{theorem}
\crefalias{enumiii}{theorem}

\newcommand{\maxb}[1]{\max\{#1\}}
\newcommand{\conv}[1]{\mathrm{conv}(#1)}
\newcommand{\nn}[1]{\mathsf{ReLU}_n(#1)}
\newcommand{\cpwl}{\mathsf{CPWL}_n}
\newcommand{\lm}{\mathsf{LM}_n}

\title{On Minimal Depth in Neural Networks}

\author{Juan L. Valerdi}
\date{}

\begin{document} 
\maketitle

\begin{abstract}
  Understanding the relationship between the depth of a neural network and its representational capacity is a central problem in deep learning theory. In this work, we develop a geometric framework to analyze the expressivity of ReLU networks with the notion of depth complexity for convex polytopes. The depth of a polytope recursively quantifies the number of alternating convex hull and Minkowski sum operations required to construct it. This geometric perspective serves as a rigorous tool for deriving depth lower bounds and understanding the structural limits of deep neural architectures.

  We establish lower and upper bounds on the depth of polytopes, as well as tight bounds for classical families. These results yield two main consequences. First, we provide a purely geometric proof of the expressivity bound by Arora et al. (2018), confirming that $\lceil \log_2(n+1)\rceil$ hidden layers suffice to represent any continuous piecewise linear (CPWL) function. Second, we prove that, unlike general ReLU networks, convex polytopes do not admit a universal depth bound. Specifically, the depth of cyclic polytopes in dimensions $n \geq 4$ grows unboundedly with the number of vertices. This result implies that Input Convex Neural Networks (ICNNs) cannot represent all convex CPWL functions with a fixed depth, revealing a sharp separation in expressivity between ICNNs and standard ReLU networks.
\end{abstract}

\section{Introduction}\label{sec:intro}

A \textit{ReLU neural network} of depth $m$, or with $m$ hidden layers, is a function $f:\R^n\rightarrow\R$ expressed as
\begin{equation*} 
    f = T^{(m+1)} \circ \sigma_m \circ T^{(m)} \circ \sigma_{m-1} \circ \dots \circ T^{(2)} \circ \sigma_1 \circ T^{(1)},
\end{equation*}
where $T^{(l)}:\R^{n_{l-1}} \rightarrow \R^{n_l}$ are $m + 1$ affine transformations, $n_0 = n$ and $n_{m+1} = 1$. The activation function $\sigma_l$ corresponds to the vectorized ReLU function $\sigma_l(x_1, \ldots, x_{n_l}) = (\maxb{x_1, 0}, \maxb{x_2, 0}, \ldots, \maxb{x_{n_l}, 0})$.

ReLU neural networks are continuous piecewise linear (CPWL) functions \cite{arora2018understanding, devore2021neural}. An important problem in learning theory is to characterize which CPWL functions can be represented by a ReLU network of a given depth. Such expressivity results are relevant for quantifying depth-dependent convergence rates for various function classes \cite{daubechies2022neural, daubechies2022nonlinear, devore2021neural, yarotsky2017error} and enhancing our theoretical understanding of deep architectures \cite{arora2018understanding, telgarsky2016benefits, eldan2016power, liang2017why}.

We denote the set of all CPWL functions from $\R^n\rightarrow\R$ as $\cpwl$, and the subset of these functions that can be represented by ReLU networks of depth $m$ as $\nn{m}$.

It has been shown that $\lceil \log_2(n + 1) \rceil$ hidden layers are sufficient to represent any CPWL function.
\begin{theorem}[Arora, Basu, Mianjy and Mukherjee \cite{arora2018understanding}, Theorem 2.1] \label{thm:arora}
  \[
    \cpwl = \nn{\lceil \log_2(n + 1) \rceil}.
  \]
\end{theorem}

The minimal depth required to represent any CPWL function remains an open question. Hertrich, Basu, Di Summa, and Skutella \cite{hertrich2023towards} proposed an approach that reduces the problem to analyzing a single function.

\begin{proposition}[Hertrich, Basu, Di Summa and Skutella \cite{hertrich2023towards}, Proposition 1.6] \label{prop:maxn}
  \[
    \cpwl = \nn{m} \Longleftrightarrow \maxb{x_1, x_2, \ldots, x_n, 0} \in \nn{m}.
  \]
\end{proposition}

Following \cref{prop:maxn}, the goal is to find the minimum $m$ such that the function $\maxb{x_1, \ldots, x_n, 0}\in\nn{m}$. 

Hertrich, Basu, Di Summa, and Skutella \cite{hertrich2023towards} also conjectured that this depth is exactly $\lceil \log_2(n + 1) \rceil$, which had been proven for dimensions $n \leq 3$ \cite{devore2021neural, mukherjee2017lower}. Much of the progress toward resolving this problem has come from formulating \cref{prop:maxn} in the language of convex polytopes, i.e. the convex hull of finitely many points \cite{grunbaum2003convex, ziegler1995lectures}. Haase, Hertrich, and Loho \cite{haase2023lower} proved the conjecture for ReLU networks with integer weights. This result was later extended by Averkov, Hojny, and Merkert \cite{averkov2025on} to networks with decimal fractional weights. Furthermore, Grillo, Hertrich, and Loho \cite{grillo2025depth} proved the conjecture for $n = 4$ in networks compatible with specific polyhedral complexes. More recently, the conjecture was disproved by Bakaev, Brunck, Hertrich, Stade, and Yehudayoff \cite{bakaev2025better} via polyhedral subdivisions of the simplex; they showed that $\lceil \log_3(n - 1) \rceil + 1$ hidden layers are sufficient to represent any CPWL function.

In this work, we introduce the \textit{depth complexity} of polytopes, a geometric analogue of neural network depth, to analyze the expressivity of ReLU networks. Although this concept has been used implicitly in prior work \cite{hertrich2023towards,haase2023lower,bakaev2025better}, we provide a systematic study of it.

We denote the depth complexity of a polytope $P$ as $d(P)$ and define it recursively as follows. If $P$ consists of a single point, we define $d(P) = 0$. If $P$ is not a single point, we define $d(P) = m$, where $m$ is the smallest positive integer such that
\begin{equation} \label{eq:depth}
  P = \sum_{i = 1}^q \conv{P_{i1}, P_{i2}}, \quad \text{with } d(P_{ij}) < m \quad \forall i, j.
\end{equation}
Here, $\conv{P_1, P_2} = \conv{P_1 \cup P_2}$ and $P_1 + P_2 = \{a_1 + a_2: a_i\in P_i\}$ are the \textit{convex hull} and \textit{Minkowski sum} operations respectively. 

At its core, the depth complexity of a polytope quantifies the number of alternating steps of convex hulls and Minkowski sums required to construct it. This serves as a measure of how ``complex'' it is to represent a polytope using these two operations. Under this measure, the simplest polytope is a point and depth-1 polytopes correspond to zonotopes, which are Minkowski sums of segments.

Before connecting the depth of polytopes and ReLU networks, we introduce some terminology. A \textit{linear max function} is a function of the form $f(x) = \maxb{a_1\cdot x, \ldots, a_p\cdot x}$ with $a_i\in\R^n$ for all $i$. Notably, linear max functions are convex and \textit{positively homogeneous}, meaning they satisfy $f(\lambda x) = \lambda f(x)$ for $\lambda \geq 0$.

Let $\lm$ denote the collection of linear max functions from $\R^n\rightarrow\R$, and let $\mathcal{P}_n$ represent the set of convex polytopes in $\R^n$. These two sets constitute semirings with the \textit{max} and \textit{sum} operations, and \textit{convex hull} and \textit{Minkowski sum} operations, respectively.

For the function $f(x) = \maxb{a_1\cdot x, \ldots, a_p\cdot x}$, its \textit{Newton polytope} is defined as
\[
  \mathcal{N}f := \conv{a_1, \ldots, a_p}.
\]
Similarly, for the polytope $P = \conv{a_1, \ldots, a_p}$, the associated \textit{support function} is
\[
  \mathcal{F}P(x) := \maxb{a_1\cdot x, \ldots, a_p\cdot x}.
\]
The mappings $\mathcal{N}$ and $\mathcal{F}$ are isomorphisms between the semirings ($\lm$, max, $+$) and ($\mathcal{P}_n$, conv, $+$) \cite{maclagan2021introduction, zhang2018tropical}. 

The following result establishes a depth connection between ReLU networks and polytopes.

\begin{theorem}[Hertrich, Basu, Di Summa and Skutella \cite{hertrich2023towards}, Theorem 5.2] \label{thm:relu-poly}
  Let $f\in\cpwl$ be a positively homogeneous function. Then, $f\in\nn{m}$ if and only if there exist $f_i\in \lm, i = 1,2, $ such that $f = f_1 - f_2$ and $d(\mathcal{N}f_i) \leq m$.
\end{theorem}

Combining \cref{prop:maxn} with \cref{thm:relu-poly} gives an equivalent geometric statement of \cref{prop:maxn}.

\begin{theorem} \label{thm:cpwl-poly}
  \[
    \cpwl = \nn{m}
  \]
  if and only if there exist polytopes $P_i, i = 1,2, $ such that
  \[
    P_1 = \mathcal{N}(\maxb{x_1, x_2, \ldots, x_n, 0}) + P_2 \quad \text{and} \quad d(P_i) \leq m.
  \]
\end{theorem}

\noindent To distinguish the different notions of depth, for $P\in\mathcal{P}_n$ define
\[
  d_{\mathrm{ReLU}}(P):=\min\{m:\mathcal{F}P\in\nn{m}\},
\]
the minimum depth of an unrestricted ReLU network representing the support function of $P$. By \cref{thm:relu-poly} and the identity $\mathcal{F}(P+Q)=\mathcal{F}P+\mathcal{F}Q$, we get
\begin{equation*}
  d_{\mathrm{ReLU}}(P) =\min_{Q\in\mathcal{P}_n}\maxb{d(P+Q),d(Q)},
\end{equation*}
and if we take $Q$ as the point $Q = \{(0, \ldots, 0)\}$, then
\begin{equation} \label{eq:drelu<=d}
  d_{\mathrm{ReLU}}(P) \leq d(P).
\end{equation}

To fully leverage \cref{thm:cpwl-poly}, it is relevant to establish foundational results on the depth complexity of polytopes. In addition, two natural questions arise: Is $d(P)$ bounded among polytopes of a fixed dimension? And what is the depth complexity of simplices such as $\mathcal{N}(\maxb{x_1, x_2, \ldots, x_n, 0})$?\\

\noindent\textit{Depth Complexity Results}\\

In \cref{sec:upper}, we derive depth complexity upper bounds based on Minkowski sums, affine transformations, convex hull representations, and combinatorial data such as the number of vertices, edges, and 2-faces. By \eqref{eq:drelu<=d}, an upper bound on $d(P)$ also gives an upper bound on $d_{\mathrm{ReLU}}(P)$ of $\mathcal{F}P$.

An affine transformation is a mapping $\varphi(x) = Ax + b$, where $A$ is a linear map and $b$ is a translation vector. A convex hull representation of a polytope $P$ in terms of other polytopes $P_1, \ldots, P_k$ is a decomposition of the form $P = \conv{P_1, P_2, \ldots, P_k}$. 

If the polytopes $P_i$ are the vertices, edges or 2-faces of $P$, we can compute depth estimates based on their counts. For a polytope with $k$ vertices, we obtain the upper bound $\lceil \log_2 k\rceil$, which is tight for certain families of polytopes.

In \cref{sec:lower}, we derive lower bounds on depth complexity using its graph and face structure. Let $G(P)$ denote the 1-skeleton of $P$, i.e. the graph whose vertices are those of $P$ and in which two vertices are adjacent when they span an edge of $P$.

We show that if $G(P)$ contains a complete subgraph on $k$ vertices, then $d(P) \geq \lceil \log_2 k\rceil$. This follows from the fact that complete subgraphs are pervasive in expressions like (\ref{eq:depth}): if $G(P)$ contains a complete subgraph, then at least a summand in the Minkowski sum will have one, which then propagates through the polytopes in the convex hull. 

Moreover, for any face $F$ of $P$, we have $d(F) \leq d(P)$. Hence, $d(P)$ is at least the maximum depth complexity among its faces. 

In \cref{sec:polytopes}, we determine the depth complexity of several families of polytopes, including pyramids, bipyramids, prisms, simplices, and cyclic polytopes. Using depth estimates based on vertex counts and complete subgraphs, it follows directly that polytopes with $k$ vertices whose graphs are complete, such as simplices and cyclic polytopes, have depth complexity $\lceil \log_2 k\rceil$.

This result answers both motivating questions posed earlier. For simplices, the main application follows from \cref{thm:cpwl-poly}, where we recover the same bound as in \cref{thm:arora}, thus providing an alternative geometric proof. In contrast, for dimension $n\geq 4$, cyclic polytopes form a family whose polytope depth grows unbounded as the number of vertices increases. This shows that there is no universal upper depth bound to represent polytopes--unlike the case for ReLU networks in \cref{thm:arora}. Notably, the proofs of these results rely only on elementary polyhedral theory.

We conclude \cref{sec:polytopes} by showing that, in dimensions $n\geq 5$, Minkowski sums of suitably chosen cyclic polytopes with zonotopes yield, for any prescribed depth $m$, families of polytopes with an unbounded number of vertices, all of depth $m$.\\

\noindent\textit{Implications for Input Convex Neural Networks}\\

Input Convex Neural Networks (ICNNs) \cite{amos2017input} are ReLU networks constrained to use only monotone affine transformations and skip connections from the input layer. Here, a monotone affine transformation means an affine map $Ax + b$ where the matrix $A$ has nonnegative entries.

ICNNs are interesting because they can represent any convex CPWL function \cite{bakaev2025depth, gagneux2025convexity} and have many successful applications \cite{lemaire2024new, xing2024optimization, bunning2021input, chen2018optimal}.

Under the ICNN model, we can define a corresponding depth complexity $d_0(P)$ for a polytope $P$. We set $d_0(P) = 0$ if $P$ is a single point; otherwise, $d_0(P)$ is the minimum $m$ such that
\[
  P = \sum_{i = 1}^p \conv{P_{i1}, P_{i2}}, \quad \text{with } d_0(P_{i1}) < m \text{ and } d_0(P_{i2}) = 0 \quad \forall i.
\]

For a detailed derivation of this definition, see \cite{bakaev2025depth}. The quantity $d_0(P)$ is exactly the minimum ICNN depth required to represent $\mathcal{F}P$. In \cite{bakaev2025depth}, the authors showed that $d_0(\mathcal{N}(\maxb{x_1, \ldots, x_n})) = n$, and that for $n=3$ there exists a family of polytopes with unbounded $d_0$-complexity.

Since every construction allowed in the definition of $d_0$ is also allowed in the definition of $d$, we have
\[
  d_{\mathrm{ReLU}}(P)\leq d(P)\leq d_0(P).
\]
The three quantities have different meanings. The quantity $d_{\mathrm{ReLU}}(P)$ is the minimum depth of an unrestricted ReLU network representing $\mathcal{F}P$. The value $d(P)$ is the minimum depth of a representation of $P$ using convex hull and Minkowski sum operations, while $d_0(P)$ further restricts such representations by requiring one input in each convex hull to be a point.

As we will show, cyclic polytopes with $k$ vertices satisfy $d(P) = \lceil \log_2 k \rceil$ for $n \geq 4$, and therefore form an unbounded depth family for $d_0$. Thus, although ICNNs can represent every convex CPWL function, no fixed ICNN depth suffices in general. This contrasts with the universal ReLU depth bound established in \cref{thm:arora}.\\

\textbf{Acknowledgements.} I am deeply grateful to Francisco Santos for \cref{thm:graph}, valuable discussions, and his hospitality at the University of Cantabria. I also thank Ansgar Freyer for discussions on the depth complexity of simplices.

\section{Depth Upper Bounds} \label{sec:upper}

We begin by computing basic bounds for Minkowski sums, convex hulls and affine maps. 

\begin{proposition}\label{prop:basic_upper} 
  Let $P_i, i = 1, 2,$ be polytopes with $d(P_i) \leq m_i$ and $\varphi$ be an affine transformation. Then,
  \begin{enumerate}[label=(\alph*)]
  \item $d(P_1 + P_2) \leq \maxb{m_1, m_2}$
  \item $d(\conv{P_1, P_2}) \leq \maxb{m_1, m_2} + 1$
  \item $d(\varphi(P_1)) \leq d(P_1)$
  \end{enumerate}
\end{proposition}
\begin{proof}
  These depth bounds follow immediately from the depth complexity definition. If $d(P_i) \leq m_i$, it implies that we can express $P_i$ as 
  \[
    P_i = \sum_{j = 1}^{q_i} \conv{Q_{j, i}, R_{j, i}},\quad d(Q_{j, i}), d(R_{j, i}) < m_i \quad \forall i, j.
  \]
  Then,
  \[
    P_1 + P_2 = \sum_{j = 1}^{q_1} \conv{Q_{j, 1}, R_{j, 1}} + \sum_{j = 1}^{q_2} \conv{Q_{j, 2}, R_{j, 2}}.
  \]
  All $Q_{j, i}$ and $P_{j, i}$ have depth strictly less than $\maxb{m_1, m_2}$, therefore $d(P_1 + P_2) \leq \maxb{m_1, m_2}$.

  Similarly, since $d(P_i) \leq \maxb{m_1, m_2}$, it follows by definition that $$d(\conv{P_1, P_2}) \leq \maxb{m_1, m_2} + 1.$$
  Let $\varphi(x) = Ax + b$ be an affine map, if $d(P_1) = 0$, it is trivial that $d(\varphi(P_1)) = 0$. For the purpose of induction, assume the statement is true up to $m - 1$. If $d(P_1) = m$, then
  \[
    \varphi(P_1) = \varphi\Big(\sum_{j = 1}^{q_1} \conv{Q_{j, 1}, R_{j, 1}}\Big) = \sum_{j = 1}^{q_1} \conv{AQ_{j, 1}, AR_{j, 1}} + b, 
  \]
  and by induction hypothesis $d(AQ_{j, 1}), d(AR_{j, 1}) < m$, implying $d(\varphi(P_1)) \leq d(P_1)$.
\end{proof}

If an affine transformation $\varphi$ is invertible, note that $d(\varphi(P)) = d(P)$. This also implies that $d(P_1 + P_2) = d(P_1)$ for $d(P_2) = 0$, which is simply a translation.

When $P = \conv{P_1, \ldots, P_k}$, iterating \cref{prop:basic_upper}(b) yields an upper bound for $d(P)$. There are several ways to iterate a binary operation on $k$ elements, corresponding to different choices of binary bracketing.

Any binary bracketing of $\conv{P_1, \ldots, P_k}$ can be encoded by a full binary tree whose leaves are labeled by $1, \ldots, k$: each internal node records one application of $\conv{\cdot,\cdot}$, and its two children correspond to the two subexpressions being combined. Conversely, every full binary tree with leaves labeled by $1, \ldots, k$ determines a binary bracketing of $\conv{P_1, \ldots, P_k}$. For a leaf $i$, let $\ell_i$ denote its depth in the tree, that is, its distance to the root.

We identify the smallest bound obtainable by iterating \cref{prop:basic_upper}(b) using Kraft's inequality \cite[Chapter 5.2]{cover2006elements}. 

\begin{theorem}[Kraft inequality for binary trees]
Let $\ell_1, \ldots, \ell_k$ be positive integers. Then, there exists a binary tree with leaves at depths $\ell_1, \ldots, \ell_k$ if and only if
\[
  \sum_{i=1}^k 2^{-\ell_i} \le 1.
\]
\end{theorem}

\begin{proposition} \label{prop:optimal_hull}
  Let $P = \conv{P_1, \ldots, P_k}$ with $d(P_i) \leq m_i$ for all $i$. Then
  \[
    d(P) \leq \Big\lceil \log_2\Big(\sum_{i=1}^k 2^{m_i}\Big)\Big\rceil.
  \]
  Moreover, this is the smallest bound that can be obtained by iterating \cref{prop:basic_upper}(b) over the polytopes $P_i$.
\end{proposition}
\begin{proof}
  For $k = 1$, the statement is immediate. Assume $k \geq 2$, fix a binary bracketing of $\conv{P_1, \ldots, P_k}$, and let $T$ be the corresponding full binary tree.

  We claim that the upper bound produced by iterating \cref{prop:basic_upper}(b) according to $T$ is
  \[
    M_T = \max_{1 \leq i \leq k}(m_i + \ell_i).
  \]
  We prove this by induction on the number of leaves of $T$. If $T$ has two leaves, the result is a direct consequence of \cref{prop:basic_upper}(b). If $T$ has more than two leaves, it means that the root has two subtrees $T_1$ and $T_2$, and applying \cref{prop:basic_upper}(b) at the root gives
  \[
    M_T = \maxb{M_{T_1}, M_{T_2}} + 1.
  \]
  By induction, each $M_{T_j}$ is the maximum of $m_i + \ell_i - 1$ over the leaves of $T_j$, because their depths inside the subtree are one less than in $T$. Hence
  \[
    M_T = \max_{1 \leq i \leq k}(m_i + \ell_i).
  \]

  Every binary tree satisfies Kraft's inequality
  \[
    \sum_{i=1}^k 2^{-\ell_i} \leq 1.
  \]
  Therefore,
  \[
    m_i + \ell_i \leq M_T \quad \forall i
  \]
  implies
  \[
    2^{m_i - M_T} \leq 2^{-\ell_i} \quad \forall i,
  \]
  and summing over $i$, we obtain
  \[
    2^{-M_T}\sum_{i=1}^k 2^{m_i} \leq 1,
  \]
  so
  \begin{equation}
    \label{eq:kraft}
    M_T \geq \Big\lceil \log_2\Big(\sum_{i=1}^k 2^{m_i}\Big)\Big\rceil.
  \end{equation}

  To show that this lower bound is attained for some $T$, set
  \[
    M = \Big\lceil \log_2\Big(\sum_{i=1}^k 2^{m_i}\Big)\Big\rceil
    \qquad\text{and}\qquad
    L_i = M - m_i.
  \]
  Therefore,
  \[
    \sum_{i=1}^k 2^{-L_i} = 2^{-M}\sum_{i=1}^k 2^{m_i} \leq 1.
  \]
  By Kraft's inequality, there exists a binary tree $T_0$ with leaves at depths $L_1, \ldots, L_k$. If $T_0$ is not full, contract every internal node with exactly one child. This produces a full binary tree $T$ with the same labeled leaves, and the depth $\ell_i$ of leaf $i$ in $T$ satisfies $\ell_i \leq L_i$ for all $i$. The tree $T$ therefore determines a binary bracketing of $\conv{P_1, \ldots, P_k}$, and for this bracketing
  \[
    M_T = \max_{1 \leq i \leq k}(m_i + \ell_i) \leq \max_{1 \leq i \leq k}(m_i + L_i) = M.
  \]
  Combined with (\ref{eq:kraft}), we get $M_T = M$, which is then the smallest upper bound obtainable by iterating \cref{prop:basic_upper}(b).
\end{proof}

If $m_i \leq m$ for all $i$ in \cref{prop:optimal_hull}, then
\[
  d(P) \leq \Big\lceil \log_2\Big(\sum_{i=1}^k 2^{m_i}\Big)\Big\rceil \leq \lceil \log_2 k \rceil + m.
\]
We will use this simplified form to obtain depth estimates when the numbers of low-dimensional faces of a polytope are known.

For this, we also need the following concepts: $\rho(P)$ and $\rho_2(P)$ are the minimum number of edges and 2-faces, respectively, needed to cover all vertices of the polytope $P$; $\nu(P)$ is the size of a maximum matching of the graph $G(P)$, i.e. the largest set of edges in $G(P)$ with no common vertices; $P_x$ is the vertex figure of $P$ at a vertex $x$, i.e. the intersection of $P$ with a hyperplane that separates $x$ from the other vertices of $P$.

\begin{proposition} \label{prop:f}
  Let $P$ be an $n$-dimensional polytope with $n\geq2$, and let $f_0, f_1, f_2$ be its numbers of vertices, edges, and $2$-faces, respectively. Then we have the following depth bounds in terms of each $f_i$ individually:
  \begin{enumerate}[label=(\alph*)]
  \item $\text{For } f_0:\quad d(P) \leq \lceil \log_2 f_0 \rceil$
  \item $\text{For } f_1:\quad d(P) \leq \Big\lceil\log_2\Big(\Big\lfloor\frac{2f_1}{n} - \min\Big(n, \Big\lfloor \frac{1+\sqrt{1+8f_1}}{4} \Big\rfloor\Big)\Big\rfloor\Big)\Big\rceil + 1$
  \item $\text{For } f_2:\quad d(P) \leq \Big\lceil \log_2 \Big(f_2 - \binom{n}{2} + 1\Big)\Big\rceil + 2$
  \end{enumerate}
\end{proposition}
\begin{proof}
  Any polytope $P$ can be expressed as $P = \conv{P_1, \ldots, P_{f_j}}$, where $P_i$ are $j$-dimensional faces of $P$.\\

  (a) If $P_i$ are the vertices, then $d(P_i) = 0$, and by \cref{prop:optimal_hull} we obtain  
  \[
    d(P) \leq \lceil \log_2 f_0 \rceil.
  \]

  (b) If $P_i$ are the edges of $P$, then $d(P_i) = 1$. We do not need all $f_1$ edges as $P$ can be written as  
  \[
    P = \conv{P_{i_1}, \ldots, P_{i_{\rho(P)}}},
  \]  
  where $\rho(P)$ is the size of a minimum edge cover $P_{i_1}, \ldots, P_{i_{\rho(P)}}$. 

  To estimate $\rho(P)$, we use standard relations between $\rho(P)$, $\nu(P)$, $f_0$, and $f_1$ (see \cite[Chapter 3.1]{west2001introduction}):  
  \[
    \rho(P) + \nu(P) = f_0, 
    \qquad 
    \nu(P) \geq \min\big(n, \lfloor f_0/2 \rfloor\big).
  \]  
  Furthermore, the degree sum formula and the complete graph vertices bound yield  
  \[
    nf_0 \leq 2f_1 \leq f_0(f_0 - 1).
  \]  
  Combining these inequalities gives  
  \[
    \rho(P) \leq \Big\lfloor \frac{2f_1}{n} - \min\Big(n, \Big\lfloor \tfrac{1+\sqrt{1+8f_1}}{4} \Big\rfloor\Big)\Big\rfloor,
  \]  
  from which the stated bound on $d(P)$ in terms of $f_1$ follows.
  
  (c) Consider the case of $2$-faces. In \cref{sec:polytopes} we show that the depth complexity of polygons is at most $2$. Hence, similarly as (b), it suffices to bound $\rho_2(P)$, the minimum number of $2$-faces covering all vertices of $P$, and apply \cref{prop:optimal_hull}.  

  Let $P_x$ be the vertex figure of $P$ at a vertex $x$. Since $\dim P_x = n-1$, the degree sum formula implies that $P_x$ has at least $\binom{n}{2}$ edges. These correspond bijectively to the $2$-faces of $P$ incident to $x$. Thus, even after removing any $\binom{n}{2} - 1$ $2$-faces, every vertex of $P$ is still contained in at least one $2$-face. Hence $\rho_2(P) \leq f_2 - \binom{n}{2} + 1$, which yields  
  \[
    d(P) \leq \Big\lceil \log_2 \big(f_2 - \binom{n}{2} + 1\big)\Big\rceil + 2. \qedhere
  \]
\end{proof}

\section{Depth Lower Bounds} \label{sec:lower}

In this section, we study depth lower bounds for polytopes, obtained from their graphs and from the depth complexity of their faces.

\begin{proposition} \label{prop:faces} 
  Let $P$ be a polytope. Then for any nonempty face $F$ of $P$, we have $d(P) \geq d(F)$.
\end{proposition}
\begin{proof}
  If $d(P) = 0$, there is nothing to prove. If $d(P) = 1$, then $P$ is a zonotope, and any face $F$ is also a zonotope, hence $d(F) \leq 1$.

  Now suppose, for induction, that the statement holds for all polytopes of depth at most $m-1$, and assume $d(P) = m$. By definition, we can write
  \[
    P = \sum_{i=1}^q \conv{P_{i1}, P_{i2}}, \quad d(P_{ij}) < m \quad \forall i, j.
  \]
  
  Let $F$ be a face of $P$. Then $F = \sum_{i=1}^q F_i$, where each $F_i$ is a face of $\conv{P_{i1}, P_{i2}}$. This implies that $F_i$ can be written as $F_i = \conv{F_{i1}, F_{i2}}$, where $F_{ij}$ is a (possibly empty) face of $P_{ij}$. 

  By the induction hypothesis, whenever $F_{ij} \neq \emptyset$, we have $d(F_{ij}) < m$. If both $F_{i1}$ and $F_{i2}$ are nonempty, then by \cref{prop:basic_upper}(b) we have $d(F_i) \leq m$. If exactly one of $F_{i1}, F_{i2}$ is nonempty, then $F_i$ is equal to that nonempty face, and hence $d(F_i) \leq m$. 

Therefore, in all cases each $F_i$ satisfies $d(F_i) \leq m$. Finally, since $F = \sum_{i=1}^q F_i$, \cref{prop:basic_upper}(a) implies $d(F) \leq m = d(P)$.
\end{proof}

To obtain a lower bound from complete subgraphs contained in $G(P)$, we first state the following lemma.

\begin{lemma}\label{lem:subgraph}
  If the graph of a polytope $G(P)$ contains a complete subgraph with $k\geq 3$ vertices, and $P$ can be decomposed as $P = \sum_{i=1}^q P_i$, then at least one of $G(P_i)$ also contains a complete subgraph with $k$ vertices.
\end{lemma}
\begin{proof}
  Any edge $e=[x_1,x_2]$ of $P$ can be written as $e = \sum_{i=1}^q e_i$, where each $e_i$ is either a vertex or an edge of $P_i$. If $e_i$ is a vertex, we regard it as a degenerate edge of length zero, so we can say that each edge of $P$ decomposes as a sum of edges of the $P_i$.

  Now let $u,v,w$ be three vertices in the complete subgraph of $G(P)$ (so $[u,v],[u,w],[v,w]$ are edges of $P$). Write
  \[
    [u, v] = \sum_{i=1}^q [u_i, v_i], \quad [u, w] = \sum_{i=1}^q [u_i, w_i], \quad [v, w] = \sum_{i=1}^q [v_i, w_i],
  \]
  where $u=\sum_i u_i$, $v=\sum_i v_i$, $w=\sum_i w_i$ with $u_i,v_i,w_i$ vertices of $P_i$.  

  Each nondegenerate segment among $[u_i,v_i]$, $[u_i,w_i]$, and $[v_i,w_i]$ is parallel to the corresponding edge $[u,v]$, $[u,w]$, or $[v,w]$, since otherwise the corresponding Minkowski sum would have dimension at least two.

  If two of $u_i,v_i,w_i$ coincide, then all three must coincide. Indeed, if $u_i=v_i\ne w_i$, the nonzero vector $w_i-u_i=w_i-v_i$ would be parallel to both $w-u$ and $w-v$, contradicting the noncollinearity of $u,v,w$. The other cases are analogous.

  Thus, for each $i$, either $u_i=v_i=w_i$, or the triangle $[u_i,v_i,w_i]$ is nondegenerate and homothetic to $[u,v,w]$. In either case,
  \[
    \frac{|u_i-v_i|}{|u-v|} = \frac{|u_i-w_i|}{|u-w|} = \frac{|v_i-w_i|}{|v-w|},
  \]
  with all three ratios equal to zero precisely when $u_i=v_i=w_i$.

  As the triangle $[u, v, w]$ is not degenerate, at least one index $j$ has these ratios nonzero, so $u_j,v_j,w_j$ are distinct vertices forming a triangle in $G(P_j)$.  

  For any other vertex $z$ in the clique of $G(P)$, applying the same argument to the triangles $[u,v,z]$ and $[u,w,z]$ shows that the same index $j$ yields non-degenerate homothetic triangles $[u_j,v_j,z_j]$ and $[u_j,w_j,z_j]$. Thus, for every vertex $z$ of the clique different from $u,v,w$, there is a distinct vertex $z_j$ of $P_j$ adjacent to all three vertices $u_j,v_j,w_j$.

  If $z$ and $z'$ are two vertices of the clique different from $u,v,w$, applying the same argument to the triangle $[u,z,z']$ shows that $[z_j,z'_j]$ is also an edge of $P_j$. Thus $u_j,v_j,w_j$, together with the vertices $z_j$ corresponding to the remaining vertices of the clique, form a complete subgraph with $k$ vertices in $G(P_j)$.
\end{proof}

\begin{theorem}\label{thm:graph}
  If the graph of a polytope $G(P)$ contains a complete subgraph with $k\geq 3$ vertices, then $d(P) \geq \lceil\log_2 k\rceil$.
\end{theorem}
\begin{proof}
  Suppose a subgraph of $G(P)$ is complete and contains $k = 3$ or $k = 4$ vertices. If we assume $d(P) = 1$, then $P = \sum_{i=1}^k P_i$, where each $P_i$ is a segment. This contradicts Lemma \ref{lem:subgraph}, which implies that at least one $P_i$ must include $k$ vertices. Therefore, we conclude $d(P) \geq 2$.

  For the sake of induction, let's assume that the result holds for all cases up to $k-1$. Now, consider that $G(P)$ includes a complete subgraph consisting of $k$ vertices. By definition, we can express $P$ as
  \[
    P = \sum_{i=1}^q \conv{P_i, Q_i}, \quad \textrm{where } d(P_i), d(Q_i) < d(P). 
  \]
  According to Lemma \ref{lem:subgraph}, there exists an index $i$ for which $G(\conv{P_i, Q_i})$ also contains a complete subgraph $K$ with $k$ vertices. Without loss of generality, we can assume that $P_i$ contains at least $\lceil \frac{k}{2}\rceil$ vertices of $K$.

  If two vertices $x, y$ of $\conv{P_i, Q_i}$ span an edge, then they also span an edge in $P_i$ if $x, y \in P_i$. Hence $G(P_i)$ contains not only at least $\lceil \frac{k}{2}\rceil$ vertices of $K$ but also the complete subgraph induced by them. By the induction hypothesis we obtain
  \[
    d(P) - 1 \geq d(P_i) \geq \Big\lceil \log_2 \Big\lceil \frac{k}{2} \Big\rceil \Big\rceil = \big\lceil \log_2 k \big\rceil - 1,
  \]
  from which it follows $d(P) \geq \lceil \log_2 k \rceil$.
\end{proof}

\section{Depth of Polytopes} \label{sec:polytopes}

We begin by computing the depth complexity of polygons, whose depth is at most 2.\\

\noindent\textbf{Polygons.} $d(P) = 1$ if $P$ is a zonotope, otherwise $d(P) = 2$.  
\begin{proof}
  If $P$ is a zonotope, then $d(P)=1$. If $P$ is a triangle, then $d(P)=2$ by \cref{prop:f}(a), since a triangle is not a zonotope.

  Now suppose $P$ is neither a zonotope nor a triangle. Then, $P$ can be decomposed as $P = \sum_{i=1}^k P_i,$ where each $P_i$ is either a zonotope or a triangle \cite[Chapter 15.1]{grunbaum2003convex}. By Proposition \ref{prop:basic_upper}(a), we obtain $d(P) = 2$. 
\end{proof}

A \textit{pyramid} is a polytope of the form $P=\conv{Q,x}$, where $Q$ is the \emph{basis} and $x\in\R^n$ such that $x\not\in\text{aff }Q$, the affine hull of $Q$.\\

\noindent\textbf{Pyramids.} $d(Q) \leq d(P) \leq d(Q) + 1$, where $Q$ is a basis of $P$.  
\begin{proof}
  From \cref{prop:basic_upper}(b) we obtain $d(P) \leq d(Q)+1$. Since $Q$ is a face of $P$, \cref{prop:faces} gives $d(Q) \leq d(P)$. 
\end{proof}

A \textit{bipyramid} is a polytope of the form $P = \conv{Q, x_1, x_2}$, where $Q$ (with $\dim Q \geq 1$) is the \textit{base}, and $x_1, x_2 \notin \text{aff }Q$ are the \textit{apices}, such that the segment $[x_1, x_2]$ intersects the relative interior of $Q$.\\

\noindent\textbf{Bipyramids.} $d(Q) \leq d(P) \leq d(Q) + 1$, where $Q$ is the basis of $P$.  
\begin{proof}
  The upper bound $d(P) \leq d(Q)+1$ follows from \cref{prop:basic_upper}(b). For the lower bound, let $H$ be a hyperplane containing aff $Q$ but not $x_1,x_2$. Projecting $P$ onto $H$ along the line through the apices gives $Q$, hence by \cref{prop:basic_upper}(c) we have $d(Q) \leq d(P)$.
\end{proof}

\noindent\textbf{Cartesian product.} $d(P_1\times P_2)=\maxb{d(P_1),d(P_2)}.$
\begin{proof}
Let
  \[
    P_1' = P_1\times\{0\}, \qquad P_2' = \{0\}\times P_2.
  \]
  Then $P_1\times P_2 = P_1' + P_2'$. Since $P_1'$ and $P_2'$ are affinely equivalent to $P_1$ and $P_2$, respectively, \cref{prop:basic_upper}(a, c) gives
  \[
    d(P_1\times P_2) \leq \maxb{d(P_1),d(P_2)}.
  \]

  For the reverse inequality, choose vertices $x_i$ of $P_i$. Then $P_1\times\{x_2\}$ and $\{x_1\}\times P_2$ are faces of $P_1\times P_2$ affinely equivalent to $P_1$ and $P_2$. By \cref{prop:faces},
  \[
    d(P_i) \leq d(P_1\times P_2),\ \  i = 1, 2.
  \]
  Hence
  \[
    \maxb{d(P_1),d(P_2)} \leq d(P_1\times P_2),
  \]
  which proves the claim.
\end{proof}

A \textit{prism} is a polytope of the form $P = \conv{Q, x + Q}$, where $Q$ is called the base of $P$ and the translation $x + Q\notin\text{aff }Q$.\\

\noindent\textbf{Prisms.} $d(P) = d(Q)$, where $Q$ is the base of $P$ and $\dim Q>0$.  
\begin{proof} 
  A prism with base $Q$ can be written as $P=Q\times [0,1]$ up to affine equivalence. The claim follows from the cartesian product depth, since $d([0,1])=1$.
\end{proof}

The \textit{cross-polytope} is the polytope $P = \conv{\pm e_1, \ldots, \pm e_n}$, where $e_i$ are the standard basis vectors of $\R^n$.\\

\noindent\textbf{Cross-polytopes.} $d(P)=\lceil \log_2 n \rceil$ for $n\geq 2$.  
\begin{proof}
  The case $n=2$ is a square, hence a zonotope, so $d(P)=1$. For $n\geq3$, the vertices $e_1,\ldots,e_n$ form a clique in the graph of $P$, and \cref{thm:graph} gives
  \[
    d(P) \geq \lceil \log_2 n \rceil.
  \]

  For the upper bound, proceed by induction on $n$. Assume $n\geq 3$
  and write $n=a+b$ with $a=\lfloor n/2 \rfloor$ and
  $b=\lceil n/2 \rceil$. Identifying $\R^n = \R^a\times\R^b$, we have
  \[
    P = \conv{P_a\times\{0\},\; \{0\}\times P_b},
  \]
  where $P_a, P_b$ are the cross-polytopes of dimension $a$ and $b$, respectively. Therefore, by \cref{prop:basic_upper}(b, c), the cartesian product depth, and the induction hypothesis,
  \begin{align*}
    d(P)
    &\leq \maxb{d(P_a\times\{0\}), d(\{0\}\times P_b)} + 1 \\
    &= \maxb{d(P_a), d(P_b)} + 1 \\
    &= \lceil \log_2 b \rceil + 1 \\
    &\leq \lceil \log_2 n \rceil.
  \end{align*}
  Hence $d(P) = \lceil \log_2 n \rceil$.
\end{proof}

A \textit{2-neighborly polytope} $P$ is a polytope where the segment between any two vertices is an edge of $P$.\\

\noindent\textbf{2-neighborly polytopes.} $d(P) = \lceil \log_2 k \rceil$, where $k$ is the number of vertices of $P$.  
\begin{proof}
  Since $G(P)$ is complete, this follows from \cref{prop:f}(a) and \cref{thm:graph}.  
\end{proof}

As simplices are 2-neighborly we obtain directly their depth complexity.\\

\noindent\textbf{Simplices.} $d(P) = \lceil\log_2 (n + 1) \rceil$.\\

Using \cref{thm:cpwl-poly}, we obtain a geometric proof of \cref{thm:arora} for ReLU networks: the Newton polytope of $\maxb{x_1, \ldots, x_n, 0}$ is an $n$-simplex, so the computed depth for simplices recovers the same depth bound as \cref{thm:arora}.

The \textit{join} of polytopes is a polytope of the form $P_1 * \cdots * P_k = \conv{P_1, \ldots, P_k}$, where the $P_i$ lie in pairwise skew affine subspaces.\\

\noindent\textbf{Join of $(2^{m_i}-1)$-simplices.} $d(P) = \Big\lceil \log_2\Big(\sum_{i=1}^k 2^{m_i}\Big)\Big\rceil$, for $P = P_1 * \cdots * P_k$, where each $P_i$ is a $(2^{m_i}-1)$-simplex.
\begin{proof}
  For each $i$, the simplex $P_i$ has depth $d(P_i)=m_i$. Since $P = \conv{P_1, \ldots, P_k}$, \cref{prop:optimal_hull} gives
  \[
    d(P) \leq \Big\lceil \log_2\Big(\sum_{i=1}^k 2^{m_i}\Big)\Big\rceil.
  \]

  On the other hand, the graph of $P$ is complete on $\sum_{i=1}^k 2^{m_i}$ vertices, since every pair of vertices from different $P_i$ is adjacent. Therefore, \cref{thm:graph} yields
  \[
    d(P) \geq \Big\lceil \log_2\Big(\sum_{i=1}^k 2^{m_i}\Big)\Big\rceil,
  \]
  proving the claim.
\end{proof}

The depth computation for join of simplices shows that for every vector $(m_1, \ldots, m_k) \in \mathbb{N}^k$, there are polytopes attaining exactly the upper bound of \cref{prop:optimal_hull}.

A \textit{cyclic polytope} $C_n(k), k > n,$ is a polytope of the form
\[
  C_n(k) = \conv{\gamma(t_1), \ldots, \gamma(t_k)},
\]
where $t_1 < \ldots < t_k$ are real numbers and $\gamma(t) = (t, t^2, \ldots, t^n)$. For $n\geq 4$, cyclic polytopes are 2-neighborly, so we obtain their depth complexity directly.\\

\noindent\textbf{Cyclic polytopes.}  $d(C_n(k)) = \lceil\log_2 k \rceil$ for $n\geq 4$.\\

Since $d(P) \leq d_0(P)$, where $d_0(P)$ is the minimum ICNN depth required to represent $\mathcal{F} P$, the depth of cyclic polytopes for $n\geq 4$ shows that ICNNs do not satisfy a result like \cref{thm:arora}. Unlike general ReLU networks, no fixed depth suffices to represent $\mathcal{F} C_n(k)$ as $k$ grows.

For $n=2$, polygons were shown to have bounded depth $d(P)\leq 2$, while the case $n=3$ remains open. Nevertheless, we show that some triangular bipyramids have depth 3. This demonstrates that depth complexity in dimension $n=3$ behaves differently from the case $n=2$.

Before the computation, we first introduce a simple tool for indecomposable polytopes. Two polytopes $P$ and $Q$ are said to be \textit{positively homothetic} if $P=\lambda Q+w$ for some $\lambda>0$ and $w\in\R^n$. A polytope $P$ is said to be \textit{indecomposable} if, in every decomposition $P=\sum_{i=1}^k P_i$, each summand $P_i$ is either a point or is positively homothetic to $P$. 

\begin{lemma} \label{lem:indecomposable}
  If $P$ is an indecomposable polytope with $d(P) > 0$, then there exist polytopes $P_1, P_2$ such that $P = \conv{P_1, P_2}$ and $d(P) = \maxb{d(P_1), d(P_2)} + 1$.
\end{lemma}
\begin{proof}
  By definition, if $d(P) = m$, then $P$ can be written as
  \[
    P = \sum_{i = 1}^q \conv{P_{i1}, P_{i2}}, \quad d(P_{ij}) < m \quad \forall i,j,
  \]
  where $m$ is the smallest integer for which such a representation exists. For each $i$, set $Q_i=\conv{P_{i1},P_{i2}}$. By \cref{prop:basic_upper}(b),
  \[
    d(Q_i)\leq \maxb{d(P_{i1}),d(P_{i2})}+1\leq m.
  \]
  At least one $Q_i$ has depth $m$; otherwise, by \cref{prop:basic_upper}(a), we would have $d(P)\leq m-1$. Fix such an index $i$, and since $d(Q_i)=m$, the preceding upper bound gives
  \[
    \maxb{d(P_{i1}),d(P_{i2})}+1=d(P).
  \]
  This $Q_i$ is not a point, so indecomposability gives $\lambda_i>0$ and $w_i\in\R^n$ such that
  \[
    P = \lambda_i \conv{P_{i1}, P_{i2}} + w_i
      = \conv{\lambda_i P_{i1} + w_i, \lambda_i P_{i2} + w_i}.
  \]
  Each $\lambda_i P_{ij}+w_i$ is an affine image of $P_{ij}$ under an
  invertible map. By \cref{prop:basic_upper}(c), affine equivalence
  preserves depth, so
  \[
    d(P) = \maxb{d(P_{i1}), d(P_{i2})} + 1
      = \maxb{d(\lambda_i P_{i1}+w_i),
          d(\lambda_i P_{i2}+w_i)} + 1. \qedhere
  \]
\end{proof}

A \textit{centered triangular bipyramid} is a bipyramid with triangular base such that the midpoint of its apices is the centroid of the base.\\

\noindent\textbf{Centered Triangular Bipyramids.}  $d(P) = 3$.
\begin{proof}
  Let $Q$ be the base triangle of $P$. Since $Q$ is a triangle, the general bipyramid bounds give $2\leq d(P)\leq3$.
  
  Suppose that $d(P)=2$. A triangular bipyramid is indecomposable \cite[Chapter 15.1]{grunbaum2003convex}, so \cref{lem:indecomposable} implies the existence of two zonotopes $Z_1, Z_2$ such that $P=\conv{Z_1,Z_2}$. Every vertex of $P$ belongs to $Z_1\cup Z_2$, so one of these zonotopes, say $Z_1$, contains at least three of the five vertices of $P$.

  The zonotope $Z_1$ cannot contain two vertices of $Q$ and one apex, or four vertices, since in either case it contains the vertices of a triangular facet of $P$, which would also be a face of $Z_1$, whereas every face of a zonotope is centrally symmetric. Nor can $Z_1$ contain all three vertices of $Q$: projection along the apices line onto $\operatorname{aff}(Q)$ would yield $Q$, contradicting the fact that projections of zonotopes are zonotopes.

  The only remaining possibility is that $Z_1$ contains both apices $a_1, a_2$ and one vertex of $Q$. In this case, both apices are the singleton maximum and minimum faces of $Z_1$ under any linear functional whose vector is normal to $\operatorname{aff}(Q)$ within $\operatorname{aff}(P)$. This implies that the apices are images of one another under the central symmetry of $Z_1$. Thus, if $c$ is the center of $Z_1$, then $a_2=2c-a_1$, and hence $c=(a_1+a_2)/2$.

  Let $v_1,v_2,v_3$ be the vertices of $Q$, and assume $v_1\in Z_1$. Since $Z_1$ is centrally symmetric about the midpoint $c$, then $Z_1$ also contains the reflection $v$ of $v_1$ through this midpoint, $v = a_1 + a_2 - v_1$. As $P$ is a centered triangular bipyramid we have that
  \[
    \frac{a_1+a_2}{2}=\frac{v_1+v_2+v_3}{3},
  \]
  therefore,
  \[
    v=a_1+a_2-v_1=-\frac{1}{3}v_1+\frac{2}{3}v_2+\frac{2}{3}v_3.
  \]
  This point lies in $\operatorname{aff}(Q)$ but outside $Q$, since this is the unique representation of $v$ as an affine combination of $v_1,v_2,v_3$ with a negative coefficient. The point $v$ is therefore outside $P$, contradicting $Z_1\subseteq P$. Thus $d(P)\neq2$, and consequently $d(P)=3$.
\end{proof}

We conclude with a complementary construction: whereas for cyclic polytopes the depth grows with the number of vertices, here we obtain polytopes with arbitrarily many vertices at fixed depth.

\begin{theorem} \label{thm:fixed}
  For every pair of integers $n \geq 5$ and $m\geq1$, there exists a family of $n$-polytopes with an unbounded number of vertices, every member of which has depth $m$.
\end{theorem}
\begin{proof}
  Let $g_i=[\mathbf{0},b_i]$, $i=1,\ldots,k$, be line segments from the origin, where $b_1,\ldots,b_k\in\R^n$ are in general position and $k\geq n$. The zonotope $Z_k=\sum_{i=1}^k g_i$ has depth one and $2\sum_{i=0}^{n-1}\binom{k-1}{i}$ vertices \cite{zaslavsky1975facing}, which tend to infinity with $k$. This proves the result when $m=1$.

  Let $m\geq2$, and let $H_k$ be a supporting hyperplane of a vertex $z_k$ of $Z_k$, exposed by a vector $u_k$. Take a polytope $Q_k\subset H_k$ with $d(Q_k)=m$: for $m=2$, take a triangle; for $m\geq3$, take the 4-dimensional cyclic polytope with $2^m$ vertices. The exact depth computations for polygons and cyclic polytopes show that $d(Q_k)=m$. Such a choice of $Q_k\subset H_k$ is possible because $n\geq5$, and affine invariance preserves depth.

  Let $P_k:=Q_k+Z_k$. It satisfies $d(P_k)\leq m$ by \cref{prop:basic_upper}(a), and the face $P_k$ in direction $u_k$ is $Q_k+z_k$, a translation of $Q_k$, thus by \cref{prop:faces} $m=d(Q_k)\leq d(P_k)$. Therefore, $d(P_k)=m$.

  Finally, for each vertex $z$ of $Z_k$, choose a generic vector in the interior of its normal cone. Here generic means that it avoids the finitely many lower-dimensional cones of the normal fan of $Q_k$, and such a choice is possible because they cannot contain the full-dimensional normal cone of $z$. The chosen vector exposes $z$ in $Z_k$ and a vertex $q$ of $Q_k$, and hence exposes the vertex $q+z$ of $P_k$. Distinct vertices $z$ give distinct vertices $q+z$, since vertices of a Minkowski sum have unique decompositions into vertices of its summands. Hence the number of vertices $f_0(P_k)\geq f_0(Z_k)$, which tends to infinity with $k$.
\end{proof}

Similar constructions exist for $n=3,4$ and $m\leq3$, but extending
\cref{thm:fixed} to these dimensions for arbitrary $m$ would require a better
understanding of depth complexity in dimension $3$.

\bibliographystyle{plain}
\bibliography{bib}

\end{document}